\documentclass[10pt,twocolumn,letterpaper]{article}

\usepackage{iccv}
\usepackage{times}
\usepackage{epsfig}
\usepackage{graphicx}
\usepackage{amsmath}
\usepackage{amssymb}
\usepackage{authblk}
\usepackage{amsthm}
\usepackage{etoolbox}
\usepackage{xspace}
\usepackage{array,multirow}
\usepackage{boldline}
\usepackage[titlenumbered,ruled ]{algorithm2e}
\usepackage[symbol]{footmisc}
\usepackage{tablefootnote}
\usepackage[flushleft]{threeparttable}

\newtheorem{prop}{Proposition}

\newcommand{\MyMapTemplatePrefixc}[4]{\expandafter#1\csname#3#4\endcsname{#2{#4}}} 
\forcsvlist{\MyMapTemplatePrefixc {\def} {\mathcal}{c}} {A,B,C,D,E,F,G,H,I,J,K,L,M,N,O,P,Q,R,S,T,U,V,W,X,Y,Z}  

\newcommand{\MyMapTemplatePrefixtb}[5]{\expandafter#1\csname#4#5\endcsname{#2{#3{#5}}}} 
\forcsvlist{\MyMapTemplatePrefixtb {\def} {\tilde}{\mathbf}{t}} {A,B,C,D,E,F,G,H,I,J,K,L,M,N,O,P,Q,R,S,T,U,V,W,X,Y,Z}  
\forcsvlist{\MyMapTemplatePrefixtb {\def} {\tilde}{\mathbf}{t}} {0,1,a,b,c,d,e,f,g,h,i,j,k,l,m,n,o,p,q,r,s,u,v,w,x,y,z}  
\makeatletter
\newcommand\footnoteref[1]{\protected@xdef\@thefnmark{\ref{#1}}\@footnotemark}
\makeatother
\newcommand{\MyMapTemplateNoPrefix}[3]{\expandafter#1\csname#3\endcsname{#2{#3}}}
\forcsvlist{\MyMapTemplateNoPrefix {\def} {\mathbf} } {0,1,a,b,c,d,e, f, g, h, i, j, k, l, m, n, o, p, q, r, u, v, w, x, y, z} 
\forcsvlist{\MyMapTemplateNoPrefix {\def} {\mathbf} } {A,B,C,D,E,F,G,H,I,J,K,L,M,N,O,P,Q,R,S,T,U,V,W,X,Y,Z}  

\def\resp{\emph{resp.}\@\xspace}

\usepackage[pagebackref=true,breaklinks=true,letterpaper=true,colorlinks,bookmarks=false]{hyperref}

\iccvfinalcopy 

\def\iccvPaperID{3523} 

\ificcvfinal\fi

\begin{document}

\title{Semi-Supervised Learning by Augmented Distribution Alignment}

\author{Qin Wang\textsuperscript{1},  Wen Li\textsuperscript{1\thanks{The corresponding author}}, Luc Van Gool\textsuperscript{1,2} \\
\textsuperscript{1}ETH Zurich  
\textsuperscript{2}KU Leuven \\
  \texttt{\small qwang@student.ethz.ch} 
  \texttt{\small\{liwen,vangool\}@vision.ee.ethz.ch}}

\maketitle
\ificcvfinal\thispagestyle{empty}\fi

\begin{abstract}
   In this work, we propose a simple yet effective semi-supervised learning approach called Augmented Distribution Alignment. We reveal that an essential sampling bias exists in semi-supervised learning due to the limited number of labeled samples, which often leads to a considerable empirical distribution mismatch between labeled data and unlabeled data. To this end, we propose to align the empirical distributions of labeled and unlabeled data to alleviate the bias. On one hand, we adopt an adversarial training strategy to minimize the distribution distance between labeled and unlabeled data as inspired by domain adaptation works. On the other hand, to deal with the small sample size issue of labeled data, we also propose a simple interpolation strategy to generate pseudo training samples. Those two strategies can be easily implemented into existing deep neural networks. We demonstrate the effectiveness of our proposed approach on the benchmark SVHN and CIFAR10 datasets. Our code is available at \url{https://github.com/qinenergy/adanet}. 
\end{abstract}


\section{Introduction}
\label{intro}
\begin{figure}
	\centering
	{\includegraphics[width=\columnwidth]{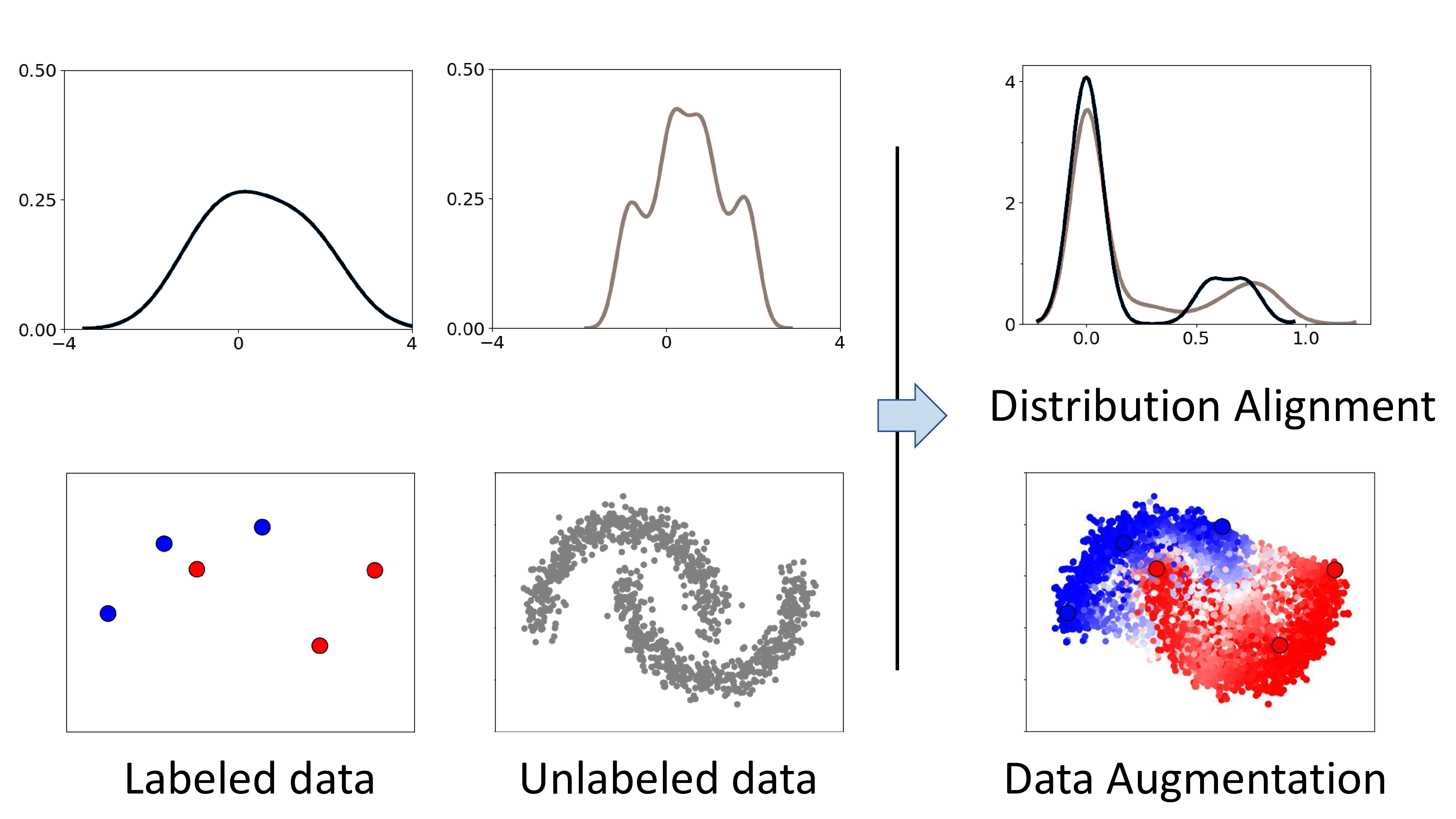}}
	\caption{Illustration of the empirical distribution mismatch between labeled and unlabeled samples with the two-moon data. The labeled and unlabeled samples are shown in the \textbf{bottom left} and \textbf{bottom middle} figures, and the kernel density estimations of  their x-axis projection are plotted in the \textbf{top left} and \textbf{top middle} figures, respectively. Our approach aims to address the empirical distribution mismatch by aligning sample distributions in the latent space (\textbf{top right}) and augmenting training samples with interpolation between labeled and unlabeled data (\textbf{bottom right}).}
	\label{fig:distmis}
	\vspace{-5pt}
\end{figure}
Semi-Supervised Learning (SSL) aims to learn a robust model with a limited number of labeled samples and a abundant number of unlabeled samples. As a classical learning paradigm, it has gained many interests from both machine learning and computer vision communities. Many approaches have been proposed in recent decades, including label propagation, graph regularization, \etc~\cite{chapelle2003cluster, blum1998combining, belkin2004semi, grandvalet2005semi, blum2001learning, zhu2005semi}. Recently, there is an increasing interest in training deep neural networks in the semi-supervised learning scenario\cite{lee2013pseudo, tarvainen2017mean, laine2016temporal, miyato2018virtual, odena2018realistic, cicek2018saas, Chen_2018_ECCV}. This is partially due to the data-intensive nature of the conventional deep learning techniques, which often impose heavy demands on data annotation and bring high cost.

While many strategies have been proposed to utilize the unlabeled data for boosting the model performance, the essential sampling bias issue in SSL has rarely been discussed in the literature. That is, \textbf{\emph{the empirical distribution of labeled data often deviates from the true samples distribution}}, due to the limited sampling size of labeled data. We illustrate this issue with the classical two-moon data in Figure~\ref{fig:distmis}, in which we plot $6$ labeled samples (bottom left) and $1,000$ unlabeled samples (bottom middle). It can be observed the two-moon structure is well depicted by the unlabeled samples. However, due to the randomness in sampling and the small sample size, it can hardly tell the underlying distribution with the labeled data, though it is also sampled from the same two-moon distribution. In terms of empirical distribution, this also leads to a considerable difference between labeled and unlabeled data, as shown by the density estimation results on their x-axis projection (top left and top middle). 

Similar empirical distribution mismatch is also observed in real world datasets for SSL (see Section~\ref{sec:ea}). As observed in domain adaptation works, the model performance can often be significantly degraded when applying on a sample set with considerable empirical distribution difference. Therefore, the SSL models could also be potentially affected by the empirical distribution mismatch between labeled and unlabeled data when exploiting different SSL strategies, \eg, label propagation from labeled data to unlabeled data. 

To tackle this issue, we propose to explicitly reduce the empirical distribution mismatch in SSL. Specifically, we develop a simple yet effective approach called Augmented Distribution Alignment. On one hand, we adopt the adversarial training strategy to minimize the distribution distance between labeled and unlabeled data, such that the feature distributions are well aligned in the latent space, as illustrated in the top right of  Figure~\ref{fig:distmis}. On the other hand, to alleviate the small sampling size issue and enhance the distribution alignment, we also propose a data augmentation strategy to generate pseudo samples by interpolating between labeled and unlabeled training sets, as illustrated in the bottom right of  Figure~\ref{fig:distmis}. It is also worth mentioning that both strategies can be implemented easily, where the adversarial training could be achieved with a simple gradient reverse layer, and the data augmentation can be implemented by interpolation. Thus, they can be readily incorporated into existing neural networks for SSL with little effort. We demonstrate the effectiveness of our proposed approach on the benchmark SVHN and CIFAR10 datasets, on which we achieve new state-of-the-art  classification performance.

Our contributions are summarized as follows:
\begin{itemize}
	\item We offer a new perspective of empirical distribution mismatch to understand semi-supervised learning. The empirical distribution mismatch problem commonly exists in SSL scenarios, however, has not been revealed by existing semi-supervised learning works. 
	\item We propose an augmented distribution alignment approach to explicitly address the empirical distribution mismatch for SSL.
    \item Our approach can be easily implemented 
    into existing neural networks for SSL with little efforts.
	\item Despite of the simplicity, our proposed approach achieves new state-of-the-art  classification performance on the the benchmark SVHN and CIFAR10 datasets for the SSL task.
\end{itemize}

\section{Related Work}

\textbf{Semi-supervised learning: } As a classical learning paradigm, many works have been proposed for semi-supervised learning with various methods, including label propagation, graph regularization, co-training, \etc\cite{yarowsky1995unsupervised, rosenberg2005semi, blum1998combining, mitchell2004role, grandvalet2005semi, blum2001learning, joachims2003transductive, ando2005framework}. We refer interested readers to \cite{zhu2005semi} for a comprehensive survey. Recently, there is an increasing interest in training deep neural networks in the semi-supervised learning scenario\cite{tarvainen2017mean, laine2016temporal, miyato2018virtual, odena2018realistic, cicek2018saas, Chen_2018_ECCV}. This is partially due to the data-intensive nature of the conventional deep learning techniques, which often impose heavy demands on data annotation and bring high cost. Different models have been designed for deep semi-supervised learning. For example, \cite{laine2016temporal, tarvainen2017mean,miyato2018virtual} proposed to add  small perturbations to unlabeled data, and enforce a consistency regularization~\cite{odena2018realistic} on the output of model. Other works~\cite{Chen_2018_ECCV,cicek2018saas} adopt the idea of self-training and used propagated labels with a memory module or regularized by training speed. The ensemble approach was also explored, where \cite{laine2016temporal} used an averaged prediction using the outputs of the network-in-training over time to regularize the model, while \cite{tarvainen2017mean} instead used accumulated parameters to for prediction. 

Different from above works, we tackle the SSL problem with a new perspective of empirical distribution mismatch, which was rarely discussed in the literature. By simply dealing with the distribution mismatch, we show that our newly proposed augmented distribution alignment with vanilla neural networks performs competitively with the state-of-the-arts SSL methods. Moreover, since we deal the SSL problem in a new way, 
our approach is potentially complementary to those approaches, and is shown to be able to further boost their performance.

\textbf{Sampling bias problem: } Sampling bias was usually discussed in the literature under the supervised learning and domain adaptation scenarios~\cite{rosset2005method,dudik2006correcting,huang2007correcting, das2018unsupervised}. Many works have been proposed to measure or address the sampling bias in the learning process~\cite{fernando2013unsupervised,ganin2014unsupervised,long2015learning, sun2015subspace}. Recently, following the generative adversarial networks~\cite{goodfellow2014generative}, the adversarial training strategy was widely used to address the empirical distribution mismatch in domain adaptation~\cite{ganin2014unsupervised, tzeng2017adversarial, yan2017learning, zhang2018collaborative, chen2018domain,gong2019dlow,chen2019learning}. Although people generally assume samples in two domains are sampled from two different distributions, while in SSL the labeled and unlabeled samples are from the identical distribution, the techniques for reducing domain distribution mismatch used in domain adaptation can be readily used to solve the empirical distribution mismatch in SSL. In this work, we employ the adversarial training strategy proposed in~\cite{ganin2014unsupervised}. 
A potential challenge as discussed in this paper is the small sample size of labeled data might lead to a lack of supports problem when aligning distribution, for which we additionally employ a sample augmentation strategy.

\textbf{Other related works:} Our work is also related to the recent proposed interpolation based data augmentation methods for training neural networks~\cite{zhang2017mixup,inoue2018data,verma2018manifold}.  In particular, the \textit{Mixup}  method~\cite{zhang2017mixup} proposed to generate new training samples using convex combinations of pairs of training samples and their labels. In order to address the small sample size issue when aligning distributions, we generalize their approach to the semi-supervised learning by using pseudo-labels for unlabeled samples in the interpolation process. Moreover, we also show that by interpolating between labeled and unlabeled data, the empirical distribution of generate data actually gets closer to the unlabeled samples.

\section{Problem Statement and Motivations}
In semi-supervised learning, we are given a small amount of labeled training samples and a large set of unlabeled training samples. Formally, let us denote by $\cD_l=\{(\x^l_{1}, y_{1}), \ldots, (\x^l_{n}, y_{n})\}$ as the set of labeled training data, where $\x^l_i$ is the $i$-th sample, $y_i$ is its corresponding label, and $n$ is the total number of labeled samples. Similarly, the set of unlabeled training data can be represented as $\cD_{u}=\{\x^{u}_1, \ldots, \x^u_m\}$ where $\x^u_i$ is the $i$-th unlabeled training sample, and  $m$ is the number of unlabeled samples. Usually  $n$ is a small number, and we have $m \gg n$. The task of semi-supervised learning is to train a classifier which performs well on the test data drawn from the same distribution with the training data.

\subsection{Empirical Distribution Mismatch in SSL}\label{sec:empd}
In semi-supervised learning, the labeled training samples $\cD_l$ and unlabeled training samples $\cD_u$ are assumed to be drawn from an identical distribution. However, due to the limited number of labeled training samples, a considerable difference of empirical distributions can often be observed between the labeled and unlabeled training samples. 

More concretely, we take the two-moon data as an example to illustrate the empirical distribution mismatch problem in Figure~\ref{fig:distmis}. In particular, the $1,000$ unlabeled samples well describe the underlying distribution (bottom middle), while the labeled samples can hardly represent the two-moon distribution (bottom left). This can be further verified by their distribution by projecting to the x-axis (upper left and upper middle), from which we observe an obvious distribution difference. Actually, when performing multiple rounds of sampling on labeled samples, the empirical distribution of labeled data varies significantly, due to the small sample number.

This phenomenon was also discussed as the sampling bias problem in the literature~\cite{gretton2012kernel,gretton2009fast, kamath2015learning}. In particular, Greton \etal~\cite{gretton2012kernel} pointed out that the difference between two samplings measured by Maximum Mean Discrepancy(MMD) depends on their sampling sizes. In semi-supervised learning where the underlying distribution of labeled and unlabeled data is assumed identical, the MMD of labeled and unlabeled data tends to vanish if and only if both sizes of two sampling are large, which is described as follows,
\begin{prop}
\label{th:1}
Let us denote $\mathcal{F}$ as a class of witness functions $f$ : $\x \rightarrow \mathcal{R}$ in the reproduced kernel Hilbert space (RKHS) induced by a kernel function $k(\dot, \dot)$, and assume $0\leq k(\dot, \dot)\leq K$, then the MMD distance of $\cD_l$ and $\cD_u$ can be bounded by $Pr\{ \text{MMD}[\mathcal{F}, \cD_l, \cD_u] > 2 (\sqrt{(K/n)} + \sqrt{(K/m)} +\epsilon)\} \leq 2\exp{\frac{-\epsilon^2nm}{2K(n+m)}}$
,\end{prop}
\begin{proof}
The proof can be derived with Theorem 7 in \cite{gretton2012kernel} by assuming the two distributions $p$ and $q$ are identical.

\vspace{-10pt}	
\end{proof}

\begin{figure}
	\centering
	{\includegraphics[width=0.85\linewidth]{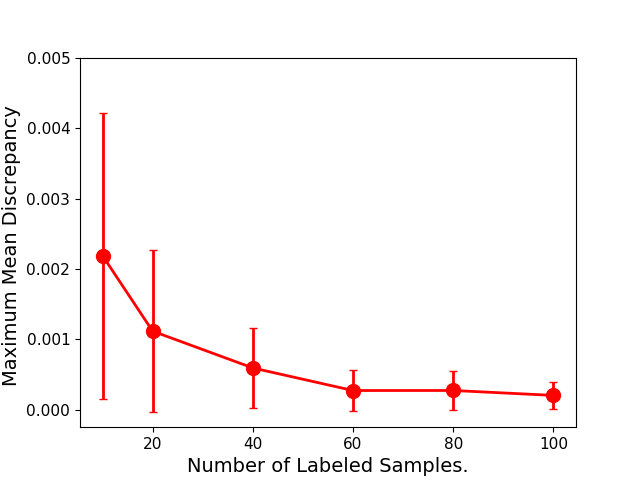}}
	\caption{MMD between labeled and unlabeled samples in two-moon example with varying number of labeled samples. Number of unlabeled sample is fixed as $1,000$.}
	\label{fig:twomoonsample}

\vspace{-10pt}	
\end{figure}

In semi-supervised learning, the number of labeled samples $n$ is usually small, which would lead to a notable empirical distribution difference with the unlabeled samples as stated in above proposition. Specifically, we illustrate the sampling bias problem with the two-moon data in the semi-supervised learning scenario in Figure~\ref{fig:twomoonsample}. We plot the MMD between labeled and unlabeled samples with regard to different numbers of labeled samples. As shown in the figure, when the sample size of labeled data is small, both the mean and variance of MMD are large, and the MMD tends to be minor only when $n$ becomes sufficiently large. 

This implies that in SSL the small sampling size often causes the empirical approximation of labeled data deviates from the true sample distribution. Consequentially, a model trained from this empirical distribution is unlikely to generalize well on the test data. While various strategies have been exploited for utilizing the unlabeled data in conventional SSL methods~\cite{lee2013pseudo, cicek2018saas, Chen_2018_ECCV}, the empirical distribution mismatch issue was rarely discussed, which is one of the hidden factors of potentially unstable problem for conventional SSL methods. This was also verified by the recent work \cite{odena2018realistic}, which shows that the performance of SSL methods could be degraded when the size of labeled dataset is decreased. 

\subsection{Healing the Empirical Distribution Mismatch}
To overcome the empirical distribution mismatch issue in SSL, in this work, we propose an augmented distribution alignment approach. In addition to training the classifier with supervision from labeled data, we also simultaneously minimize the distribution divergence between labeled and unlabeled data, such that the empirical distributions of labeled and unlabeled samples are well aligned in the latent space (as illustrated in upper right of Figure~\ref{fig:distmis}). 

Formally, let us denote the loss function as $\ell(f(\x^l_i), y_i)$ where $f$ is the classifier to be learnt. We also define $\Omega(\cD_l, \cD_u)$ as the distribution divergence of labeled and unlabeled data measured with certain metric. Then, our main idea can be formulated as the following objective,
\begin{equation}
\min_{f} \sum_{i=1}^n\ell(f(\x^l_i), y_i) + \gamma\Omega(\cD_l, \cD_u), \label{eq:ada_obj}
\end{equation}
where $\gamma$ is a trade-off parameter to balance two terms. 

An issue with the above solution is that the small number of labeled samples (\ie, $n$) potentially makes the optimization of (\ref{eq:ada_obj}) unstable. To address this issue, we further propose a data augmentation strategy. Inspired by the recent mixup approach for supervised learning, we iteratively generate new training samples by interpolating between the labeled samples and unlabeled samples, and feed them for both learning the classifier and reducing the empirical distribution divergence. We refer to our approach as Augmented Distribution Alignment, and detail it in the following section.

\section{Augmented Distribution Alignment for SSL}
In this section, we introduce our augmented distribution alignment method for SSL, in which we respectively propose two strategies, 
\emph{adversarial distribution alignment} and \emph{cross-set sample augmentation}, to tackle the empirical distribution mismatch and the small sample issues. 

\subsection{Adversarial Distribution Alignment}
\label{sec:ada}

We employ $\mathcal{H}$-Divergence~\cite{ben2010theory, cortes2011domain} to measure distribution divergence $\Omega$ as inspired by recent domain adaptation works. 
 
In particular, let us denote by $g(\cdot)$ a feature extractor (\eg, convolutional layers) which maps sample $\x$ into a latent feature space. Moreover, let $h:g(\x)\rightarrow\{0,1\}$ be a binary discriminator which predicts 0 for labeled samples and 1 for unlabeled samples. The $\cH$-Divergence between labeled and unlabeled samples can be written as:
\begin{eqnarray}
{d_\mathcal{H}}(\cD_l, \cD_u) \!\!= \!\!2 \left\{1 - \min_{h\in\mathcal{H}}\left[err(h, g, \cD_l)+err(h, g, \cD_u)\right]\right\}\nonumber,
\end{eqnarray}
where $err(h, g, \cD_l) = \frac{1}{n}\sum_{\x^l}[h(g(\x^l))\neq 0]$ is the prediction error of the discriminator $h$ on labeled samples, and $err(h, g, \cD_u)$ is similarly defined for unlabeled samples.

Intuitively, when the empirical distribution mismatch is large, the discriminator could easily distinguish the labeled and unlabeled samples, thus its prediction errors would be small, and the $\cH$-divergence is higher, and vice versa. Therefore, to reduce the empirical distribution mismatch of labeled and unlabeled samples, we then  minimize the distribution distance $d_\mathcal{H}(\cD_l, \cD_u)$ to enforce the feature extractor $g$ to generate a latent space in which two sets of features are well aligned. This is therefore achieved by solving the following problem:
\begin{eqnarray}
\min_{g}{d_\mathcal{H}}(\cD_l, \cD_u) \!\!= \!\!\max_g\min_{h\in\mathcal{H}}\left[err(h, g, \cD_l)+err(h, g, \cD_u)\right] \!\! \!\!. \nonumber
\end{eqnarray}

The above max-min problem can be optimized with the adversarial training methods. In~\cite{ganin2016domain}, Ganin and Lempitsky showed that it can be implemetned as a simple gradient reverse layer (GRL) which automatically reverse the gradient after discriminator, thus one can directly minimize the classification loss of the discriminator $h$ with the standard propagation optimization library.

\begin{figure*}
	\centering
	\includegraphics[width=0.9\linewidth]{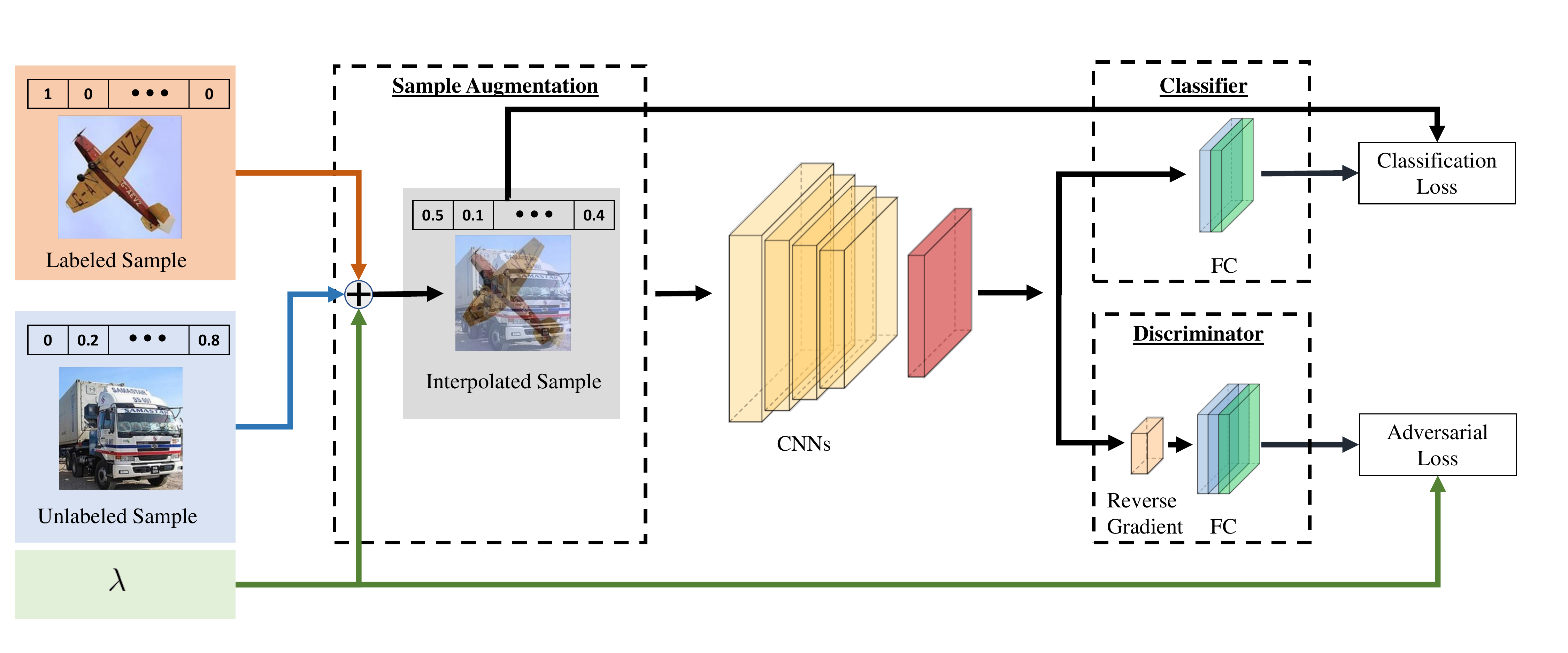}
	\caption{The network architecture of our proposed ADA-Net, in which we append an additional discriminator classifier branch with a gadient reverse layer to the vanilla CNN (shown in the bottom right part). In training time, the cross-set sample interpolation is performed between labeled and unlabeled samples, and we feed the interpolated samples into the network. Pesudo-labels of unlabeled samples are obtained using the classifier trained in last iteration (see explanation in Section~\ref{sec:summary}) for details.}
	\label{damain}
\end{figure*}

\subsection{Cross-set Sample Augmentation} 
As discussed in Section 3, in SSL, the limited sampling size of labeled data often causes unstable in optimization and leads to performance degradation.  In order to reinforce the alignment, as inspired by \cite{zhang2017mixup}, we propose to generate new training samples by interpolating between labeled and unlabeled samples. In particular, for each $\x^u$, we assign it a pseudo-label $\hat{y}^u$, which is generated by using the prediction from the model trained in previous iteration in this work. Then, given a labeled sample $\x^l$ and an unlabeled sample $\x^u$, the interpolated sample can be represented as,
\begin{eqnarray}
\tilde{\x} &=& \lambda \x^l + (1-\lambda) \x^u,  \label{eqn:interp_x}\\
\tilde{y} &=& \lambda y^l + (1-\lambda) \hat{y}^u,\label{eqn:interp_y}\\
\tilde{z} &=& \lambda \cdot 0 + (1-\lambda)\cdot1\label{eqn:interp_z},
\end{eqnarray}
where $\lambda$ is a random variable that is generated from an prior $\beta$ distribution, \ie $ \lambda \sim \beta(\alpha, \alpha)$  with $\alpha$ being a hyperparameter to control the shape of the $\beta$ distribution, $\tilde{\x}$ is the interpolated sample, $\tilde{y}$ is its class label, and $\tilde{z}$ is its label for the distribution discriminator.

The benefits of such cross-set sample augmentation are two-fold. First, the interpolated samples greatly enlarged the training data set, making the learning process more stable, especially for deep neural networks models. It was also shown in \cite{zhang2017mixup} that such data augmentation helps to improve model robustness.

Second, each pseudo-sample is generated by interpolating between a labeled sample and an unlabeled sample, thus the distribution of pseudo-samples is expected to be closer to the real distribution than that of the original labeled training samples. We prove this using the euclidean generalized energy distance~\cite{szekely2013energy} in below.

Let us denote $P_l$ and $P_u$ as the empirical distributions of labeled and unlabeled data, their euclidean generalized energy distance~\cite{szekely2013energy} can be written as,
\begin{eqnarray}
\!\!J^2(P_l,P_u)\!= \!\mathbb{E}[\|\x^l-\x^u\|^2]\!\!-\!\!\mathbb{E}[\|\x^l\!\!-{\x^l}'\|^2\!\!-\!\!\mathbb{E}[\|\x^u-{\x^u}'\|^2]\nonumber.
\end{eqnarray}
where $\|\cdot\|$ is the euclidean distance, $\x^l$ and ${\x^l}'$ (\resp, $\x^u$ and ${\x^u}'$) are two samples independent sampled from $P_l$ (\resp, $P_u$). Then, we show that cross-set sample augmentation helps to bridge the gap between two distributions by the following proposition, 
\begin{prop}
Let $\tilde{P}$ be the empirical distribution of the pseudo sample $\tilde{\x}$ generated using (\ref{eqn:interp_x}), then we have $J^2(\tilde{P},P_u) = \frac{1}{4}J^2(P_l,P_u)$. In other words, the euclidean generalized energy distance between the empirical distribution of the pseudo and unlabeled samples is smaller or equal than that of labeled and unlabeled samples. 
\end{prop}
\begin{proof}
Using Proposition 2 from~\cite{szekely2013energy}, we rewrite the energy distance $J^{2}(P_l ,P_u)$ as follows,
$$J^{2}(P_l ,P_u ) =   2\|\mathbb{E}[\x^l] - \mathbb{E}[\x^u]\|^2$$

In addition, we have 
\begin{align}   
       \mathbb{E}[\lambda \x^l+ (1-\lambda) \x^u] 
    &=  \frac{1}{2}\mathbb{E}[\x^l] + \frac{1}{2} \mathbb{E}[\x^u],\nonumber
\end{align}
because the expectation of $\lambda\sim\beta(\alpha,\alpha)$ is $0.5$, and the same applies to $1 - \lambda$. Therefore,
\begin{align}   
J^{2}(\tilde{P}, P_u) 
&=  2\|\frac{1}{2}\mathbb{E}[\x^l] + \frac{1}{2} \mathbb{E}[\x^u] - \mathbb{E}[\x^u]\|^2 \nonumber   \\  
&= 2 \|\frac{1}{2}\mathbb{E}[\x^l] - \frac{1}{2} \mathbb{E}[\x^u]\|^2  \nonumber\\
&= \frac{1}{4} J^{2}(P_l, P_u)\nonumber
\end{align}
Here we complete the proof.
\end{proof}
This implies that the new generated pseudo-samples can be deemed as being sampled from the intermediate distributions between the empirical distributions of labeled and unlabeled samples. As shown in previous domain adaptation works~\cite{gopalan2011domain,gong2012geodesic}, such intermediate distributions are beneficial to alleviate the gap between two distributions, and learn more robust models. 

\subsection{Summary}
\label{sec:summary}
\begin{algorithm}[t]
\SetAlgoLined
\SetKwInOut{Input}{Input}
\SetKwInOut{Output}{Output}
\caption{A training step for ADA-Net.}
\label{algo:alg}
\Input{A batch of labeled samples \{$(\x^l, y^l), \ldots\}$,  a batch of unlabeled samples $\{\x^u,\ldots\}$, classifier $f$ and discriminator $h$.}
\begin{enumerate}
  \item Run one forward step to get pseudo-labels for unlabeled samples, \ie, $\hat{y}^{u} \xleftarrow{} f(\x^{u})$
  \item Sample $\lambda$ of batch size from $\beta(\alpha, \alpha)$, and generate a batch of samples $\{(\tilde{\x}, \tilde{y}, \tilde{z}), \ldots\}$ using (\ref{eqn:interp_x}),(\ref{eqn:interp_y}),(\ref{eqn:interp_z}).
  \item Perform a forward pass by feeding $\{(\tilde{\x}, \tilde{y}, \tilde{z}), \ldots\}$. 
  \item Perform a backward pass by minimizing (\ref{eq:ada_obj_final}).
\end{enumerate}

\Output{classifier $f$ and discriminator $h$}
\end{algorithm}

We unify the adversarial distribution alignment and cross-set sample augmentation strategies into one framework, finally leading to our augmented distribution alignment approach. 

In Figure~\ref{damain}, we demonstrate an example of incorporating our augmented distribution alignment approach into a vanilla convolutional neural networks, which is referred to as \emph{ADA-Net}. Specifically, in addition to the classification branch, we add several fully connected layers as the discriminator to distinguish labeled and unlabeled samples (\ie, $h$ discussed in Section~\ref{sec:ada}). A gradient reverse layer is added before the discriminator, which will automatically reverse the sign of gradient from the discriminator during back-propagation. Then, for each mini-batch, we use the cross-set sample augmentation strategy in   (\ref{eqn:interp_x}),(\ref{eqn:interp_y}),(\ref{eqn:interp_z}) to generate interpolated samples and labels, and use them as training data to train our ADA-Net. The objective for training the network can be obtained by replacing the training samples and $\Omega(\cdot, \cdot)$ term in (\ref{eq:ada_obj}), \ie,
\begin{equation}
\min_{f,g, h} \sum_{\tilde{\x}}\lambda\ell(f(g(\tilde{\x})), \tilde{y}) + \gamma\ell(h(g(\tilde{\x})), \tilde{z}), \label{eq:ada_obj_final}
\end{equation}
where $g, f, h$ are respectively the feature extractor, classifier, and discriminator, and $\ell(\cdot, \cdot)$ is the loss function for which we use the cross-entropy in this work. $\lambda$ is the weight for classification loss, which corresponds to the $\lambda$ for generating the interpolated sample $\tilde{\x}$ (see (\ref{eqn:interp_x})). The reason for applying this weight is as follows. The higher $\lambda$ is, the higher proportion of $\tilde{\x}$ coming from labeled set is, and we are more confident on its label $\tilde{y}$, and vice versa.

We depict the training pipeline Algorithm~\ref{algo:alg}. Aside from the simple sample interpolation, the network can be optimized with the standard propagation approaches. Therefore, our augmented distribution alignment can be easily incorporated into existing neural networks by appending a discriminator with the GRL layer, and adding the proposed cross-set sample augmentation during mini-batch data preparation. 
 
\section{Experiments}
In this section, we evaluate our proposed ADA-Net for semi-supervised learning on benchmark datasets including SVHN, and CIFAR10.  
\subsection{Experimental Setup}
\paragraph{SVHN:} The Street View House Numbers (SVHN) dataset~\cite{netzer2011reading} is a dataset consists of real-world digit photos. It includes ten classes and 73,257 training images of 32$\times$32 size. Following~\cite{miyato2018virtual}, out of the full training set, 1000 images are used with labels for supervised learning. The rest training photos are provided without labels. Random translation is the only augmentation used for this dataset.  

\paragraph{CIFAR10:} The  CIFAR10 dataset~\cite{krizhevsky2009learning} contains 10 classes, and consists of 50,000 training images as well as 10,000 test images. All images are of the size 32$\times$32. 4,000 samples from the training images are used as labeled set for our experiments, the rest are used as unlabeled samples.  
 
We use the PreAct-ResNet-18~\cite{he2016identity} as the backbone network, and implement our ADA-Net in Tensorflow based on the open source TensorPack library~\cite{wu2016tensorpack}.  For the class classifier,  a single fully connected layer is used to map the features to logits. For the domain classifier, two dense layers, each with 1,024 units, followed by another dense layer are used to produce two channels of soft domain labels. 

The batch size is set as 128. The learning rate starts from 0.1, and is divided by 10 when 50\%, and 75\% epochs are reached. The network is trained for 100 epochs in total for SVHN, and 300 epochs for CIFAR10, where one epoch is defined as one iteration over all unlabeled data. We use a momentum optimizer with 0.9 as the momentum. The following hyperparameters are used for our reported results:  weight-decay$=0.0001$, interpolation $\alpha=0.1$ for SVHN and $\alpha=1.0$ for CIFAR10. The experiments on SVHN and CIFAR10 share the exact same network and protocol.  We followed~\cite{miyato2018virtual} to use a separate validation set of 1,000 images to select $\alpha$ for all methods. 

\subsection{Experimental Results}

We summarize the classification error rates on the SVHN and CIFAR10 dataset in Table~\ref{tab:ssl-ablation}.
We include the baseline CNN model that is trained with labeled data only as a reference. To validate the effectiveness of the two modules in our ADA-Net, we also report two variants of our proposed approach. In the first variant, we do not use cross-set sample augmentation and apply the distribution alignment using original labeled and unlabeled samples. In the second variant, we remove the discriminator and perform only cross-set sample augmentation for learning the classifier.

\begin{figure*}
    \vspace{-20pt}	
	\centering
	\includegraphics[width=0.8\linewidth]{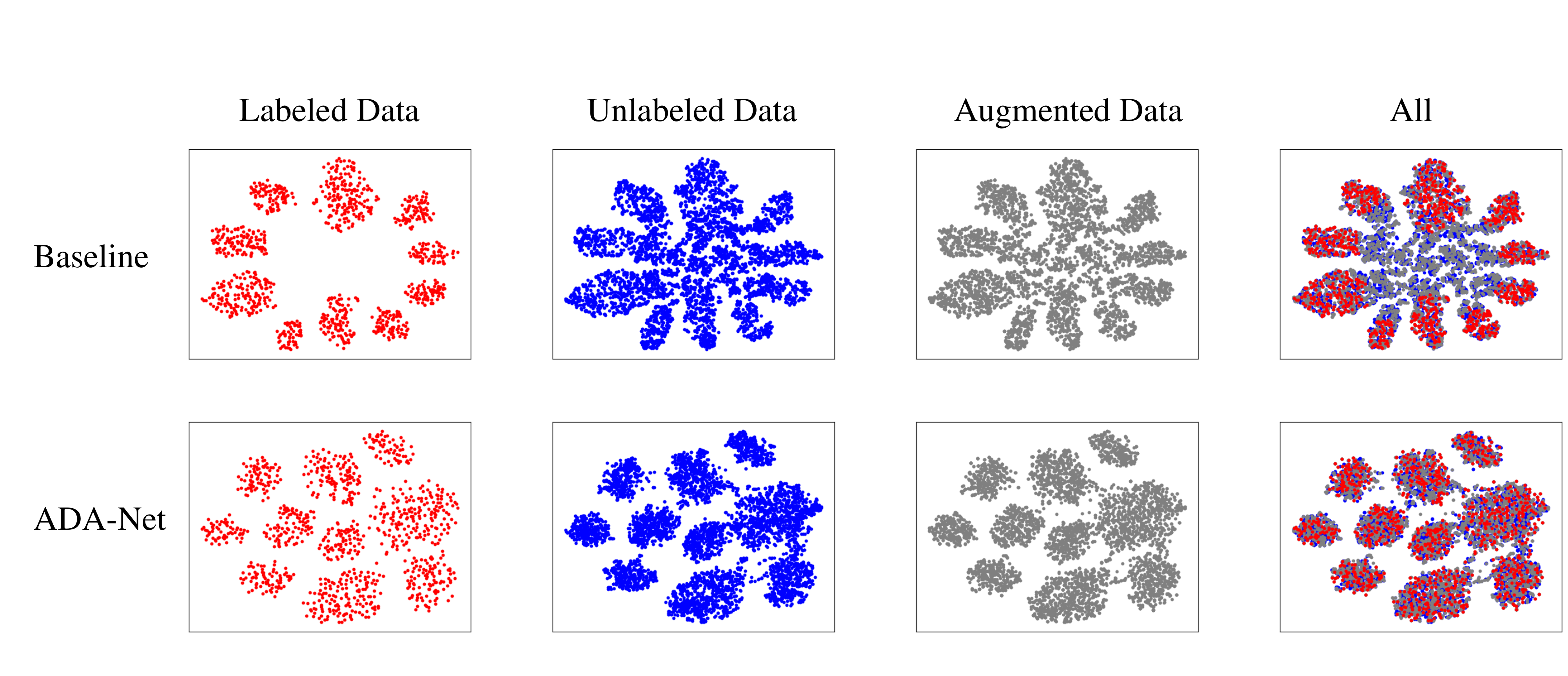}
	\caption{Visualization of SVHN features obtained by baseline CNN and our ADA-Net using t-SNE. We generated the t-SNE using labeled, unlabeled, and interpolated samples together, and show them separately for a better comparison. For baseline CNN, empirical distribution mismatch between labeled and unlabeled samples can be observed, and the augmented samples bridge the gap to some extent. For our ADA-Net, with the augmented distribution alignment, empirical distribution mismatch are well reduced.}
	\label{fig:tsne}
\end{figure*}

As shown in Table~\ref{tab:ssl-ablation}, our ADA-Net significantly improves the classification performance on both datasets. We also observe that both the distribution alignment and cross-set sample augmentation are important for improving the performance. The distribution alignment module brings 1.30\% and 3.04\% improvement on CIFAR10 and SVHN, and the cross-set sample augmentation module gives 6.18\% and 3.06\%  improvement, respectively. By integrating both modules, the classification error rates can be reduced by our ADA-Net from 19.97\% and 13.80\%  to 8.87\% and 5.90\% on the CIFAR10 and SVHN datasets, respectively. The experimental results clearly validate our motivations, and also demonstrate the effectiveness of our proposed approach.

\subsection{Experimental Analysis}
\label{sec:ea}
\paragraph{Feature visualization:} To better understand how our ADA-Net works, we use the base CNN block as a feature extractor, and visualize with the t-SNE approach for the labeled samples, unlabeled samples, and the generated pseudo-samples on the SVHN dataset in Figure~\ref{fig:tsne}. The features extracted using the baseline CNN trained with only labeled data are also visualized for comparison. As shown in Figure~\ref{fig:tsne}, a considerable distribution difference between labeled and unlabeled samples can be observed for the baseline CNN model, and the generated pseudo-samples distribute in between those two sets. Nevertheless, with our ADA-Net, the distributions of three types of samples are similar since we explicitly align the distributions of labeled and unlabeled samples in the training procedure.

\begin{table}[h!]
	\begin{center}
		\caption{Classfication error rates of our proposed ADA-Net and its variants on the CIFAR10 and SVHN datasets. ``dist" denotes the distribution alignment module, and ``aug" denotes the cross-set sample augmentation module. PreAct-ResNet-18~\cite{he2016identity} is used as the backbone network.}
		\label{tab:ssl-ablation}
		\begin{tabular}{c|c c|c|c} 
		\hlineB{3}
			 & {dist} & {aug} & {CIFAR10} & {SVHN}\\
			\hline \hline 
			Baseline & & & 19.97\% & 13.80\%  \\
			\hline
			\multirow{3}{*}{Ours} & \checkmark& & 18.67\% & 10.76\% \\
			\cline{2-5}
			& & \checkmark&   13.79\% & 10.74\%\\
			\cline{2-5}
			&\checkmark & \checkmark&   \textbf{8.87\%} & \textbf{5.90}\%\\
			\hline
		\end{tabular}
	\end{center}
\vspace{-25pt}	
\end{table}

\paragraph{Feature distribution:} To further show the effectiveness of our ADA-Net in reducing the distribution mismatch, we take the first three activations of the baseline CNN model and our ADA-Net as examples, and plot the distribution of labeled and unlabeled samples on each dimension individually. The distribution is obtained by performing kernel density estimation~\cite{parzen1962estimation, rosenblatt1956remarks} on each type of samples and each dimension individually. As shown in Figure~\ref{fig:damaind}, we again observe a considerable mismatch between the estimated empirical distribution of labeled and unlabeled samples for the baseline CNN model. And also, such distribution mismatch is then well reduced in our ADA-Net model. We have similar observations for other feature activations.

\begin{figure}[h]
    \vspace{-10pt}	
	\centering
	\includegraphics[width=0.8\linewidth]{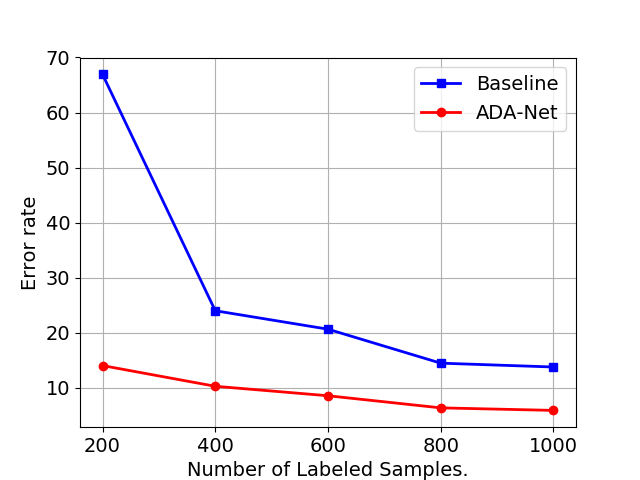}
	\caption{Classification Error rates on SVHN of our ADA-Net and baseline CNN when varying the number of labeled samples.}
	\vspace{-0.8cm}
	\label{fig:size}
\end{figure} 
\begin{figure}[t]
	\centering
	\includegraphics[width=\linewidth]{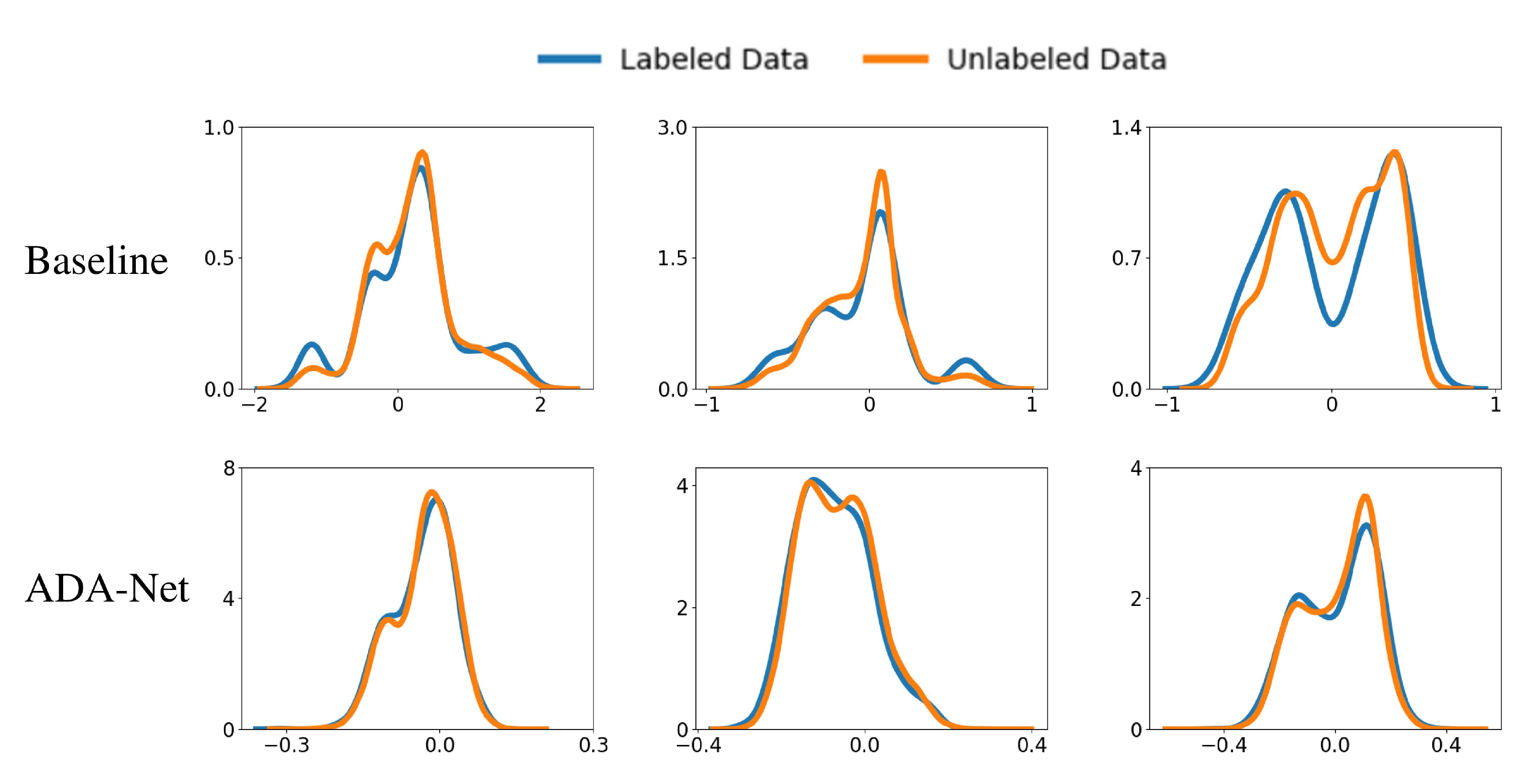}
	\caption{Kernel density estimation of labeled and unlabeled samples of the SVHN Dataset based on the first three feature activations of the baseline CNN and our ADA-Net. Considerable distribution mismatch between labeled and unlabeled data can be observed for the baseline CNN model (top row), while two distributions are generally aligned well with our ADA-Net (bottom row).}
	\label{fig:damaind}
\vspace{-10pt}	

\end{figure}

\paragraph{Varying number of labeled samples:} As discussed in Section~\ref{sec:empd}, the distribution mismatch in SSL is correlated with the number of labeled samples. It often becomes more serious when the number of labeled samples is less. To validate the effectiveness of  ADA-Net with different sampling size, we conduct experiments on the SVHN dataset by varying the number of labeled samples. In particular, we train models using $200$, $400$, $600$, $800$ and $1,000$ labeled samples, and all other experimental settings remain the same. The error rates of  ADA-Net and the baseline CNN are plotted in Figure~\ref{fig:size}. We observe that the error rate of baseline CNN model increases dramatically when reducing the number of labeled samples, which indicates that the sampling bias makes the learning problem more challenging. Nevertheless, our ADA-Net consistently improves the classification performance by alleviating such sampling bias with the augmented distribution alignment, the relative improvement is more obvious when the labeled samples are rare.

\subsection{Comparison with State-of-the-arts}

\begin{table}[h!]
	\begin{center}
		\caption{Classification error rates of different methods on CIFAR10 and SVNH. Conv-Large~\cite{ tarvainen2017mean} is used as the backbone network. Results of baseline methods are taken from the papers.}
		\label{tab:ssl}
		
\begin{threeparttable}
		\begin{tabular}{c|c|c} 
		\hlineB{3}
			\textbf{Method} & \textbf{CIFAR10} & \textbf{SVHN} \\
			\hline
			$\Pi$ Model~\cite{laine2016temporal}& 12.36\%  & 4.82\% \\
			Temporal ensembling~\cite{laine2016temporal}& 12.16\%  & 4.42\% \\
			Mean Teacher\cite{tarvainen2017mean}& 12.31\% & 3.95\%\\			
			VAT~\cite{miyato2018virtual}& 11.36\%  & 5.42\% \\
           VAT+Ent~\cite{miyato2018virtual}& 10.55\%  & 3.86\% \\			
			SaaS~\cite{cicek2018saas}& 13.22\% & 4.77\% \\
			MA-DNN~\cite{Chen_2018_ECCV}& 11.91\% & 4.21\%\\
            VAT+Ent+SNTG~\cite{luo2018smooth}& 9.89\% & 3.83\%\\
            Mean Teacher+fastSWA\tnote{*} ~\cite{athiwaratkun2018there}& 9.05\% & -\\
    		\hline
    		ADA-Net (Ours) &  10.30\% & 4.62\% \\
			ADA-Net+ (Ours) &   10.09\% & \textbf{3.54\%}\\
			ADA-Net\tnote{*} (Ours) &   \textbf{8.72\%} & - \\
			\hline 
		\end{tabular}
		
  \begin{tablenotes}
    \item[*] \scriptsize Larger translation range (4 instead of 2), and without ZCA whitening.
  \end{tablenotes}
\end{threeparttable}
    \vspace{-15pt}	
    \end{center}
	
\end{table}

\begin{table}[h!]
	\begin{center}
		\caption{Classification error rates of different methods on ImageNet dataset. ResNet-18 is used as the backbone network.}
		\label{tab:ssl-imagenet}
		\begin{tabular}{c|c|c} 
		\hlineB{3}
			\textbf{Method} & \textbf{Top-1} & \textbf{Top-5} \\
			\hline
			100\% Supervised & 30.43\%  & 10.76\% \\
			10\% Supervised & 52.23\%  &  27.54\% \\
			Mean Teacher~\cite{tarvainen2017mean}&  49.07\% & 23.59\%\\			
			Dual-View Deep Co-Training~\cite{qiao2018deep}& 46.50\% & 22.73\%\\
    		\hline
    		ADA-Net (Ours) &  \textbf{44.91}\% & \textbf{21.18}\% \\
			\hline 
		\end{tabular}
	\end{center}
	
\vspace{-10pt}	
\end{table}

We further compare our ADA-Net with recently proposed state-of-the-art SSL learning approaches~\cite{laine2016temporal,tarvainen2017mean,miyato2018virtual,miyato2018virtual,cicek2018saas,Chen_2018_ECCV,luo2018smooth, athiwaratkun2018there}. As discussed in \cite{odena2018realistic}, minor modification in the network structure and data processing method often lead to different results. To ensure a fair comparison, we take the VAT method~\cite{miyato2018virtual} as a reference, and  strictly follow their experimental setup. In particular, we re-implement our ADA-Net based on the released codes from~\cite{miyato2018virtual}. The same Conv-Large architecture and hyper-parameters are used.  

We report the results of different methods on the CIFAR10 and SVHN datasets in Table~\ref{tab:ssl}. Our ADA-Net achieves competitive results with those state-of-the-art SSL methods. Despite the simplicity of our augmented distribution alignment, the results clearly validate the importance on dealing with the empirical distribution mismatch in the semi-supervised learning, and also demonstrates the effectiveness of our ADA-Net. Furthermore, by adopting a simpler augmentation setup used by~\cite{athiwaratkun2018there}, our vanilla ADA-Net approach pushes the envelope of SSL on CIFAR10, and achieves new state-of-the-art error rates of $8.72\%$.

More importantly, as we solve the SSL problem from a new perspective that was not revealed by previous works, our augmented distribution alignment strategy is generally complementary to other methods. Therefore, the performance of existing SSL methods can be boosted by incorporating the distribution alignment and cross-set sample augmented modules proposed in this work. As shown in Table~\ref{tab:ssl}, combining our ADA-Net with the VAT+Ent method (denoted as ``ADA-Net+"), we achieve new state-of-the-art error rates of $3.54\%$ on SVHN.

We additionally report our result on 1000-class ImageNet in Table~\ref{tab:ssl-imagenet}, with 10\% labels. We compare our results  with  previous state-of-the-art methods Mean Teacher~\cite{tarvainen2017mean} and Deep Co-Training~\cite{qiao2018deep}. The result of  Deep Co-Training is quoted from their paper, and the performance of Mean Teacher is from running their official implementation by~\cite{qiao2018deep}. Following \cite{qiao2018deep}, we train ResNet-18~\cite{he2016deep} for 600 epochs with a batch size of 256, and we set $\alpha=1.0$. ADA-Net performs better than both methods and outperforms Dual-View Deep Co-Training by 1.59\% on Top-1 error rate and 1.55\% on Top-5 error rate.

\section{Conclusions}
In this work, we have proposed a new semi-supervised learning method called augmented distribution alignment. In particular, we tackle the semi-supervised learning problem from a new perspective that labeled and unlabeled data often exhibits a considerable difference in terms of the empirical distribution. We therefore employed an adversarial training strategy to align the distributions of labeled and unlabeled samples when training the neural networks. A cross-set sample augmentation was further proposed to deal with the limited sampling size and bridge the distribution gap. Those two strategies can be readily unified into the existing deep neural networks, leading to our ADA-Net. Experiments on the benchmark CIFAR10 and SVHN datasets have validated the effectiveness of our approach. 

\section*{Acknowledgements}
We thank Prof. Limin Wang, Dr. Eirikur Agustsson, Haocheng Luo, Kunjin Chen for feedback and fruitful discussions.

{\small
\bibliographystyle{ieee}
\bibliography{egbib}
}

\end{document}